\documentclass[runningheads]{llncs}
\usepackage[T1]{fontenc}
\usepackage{graphicx}
\usepackage{algorithm}
\usepackage[noend]{algorithmic}
\usepackage{amsmath}
\usepackage{amssymb}
\usepackage{pgfplots}
\usepackage{menukeys} 
\usepackage{wasysym}
\usepackage{rotating}
\usepackage{graphicx}
\usepackage{wrapfig}

\newcommand\defeq{\mathrel{\stackrel{\makebox[0pt]{\mbox{\normalfont\scriptsize def}}}{:=}}}

\newcommand{\tighttexttt}[1]{\scalebox{0.9}[1]{\texttt{#1}}}

\newcommand{\sig}[1]{{\mathrm{{#1}}}}

\newcommand{\CommentedText}[1]{}

\begin{document}
\title{Runtime Monitoring and Enforcement of Conditional Fairness in Generative AIs
}
\titlerunning{Conditional Fairness in Generative AIs}
%

\author{Chih-Hong Cheng\inst{1,2}\orcidID{0000-0002-7265-8413} \and
Changshun Wu\inst{3}\orcidID{0000-0001-8293-288} \and
Xingyu Zhao\inst{4}\orcidID{0000-0002-3474-349X} \and 
Saddek Bensalem\inst{5}\orcidID{0000-0002-5753-2126} \and
Harald Ruess\inst{6}\orcidID{0000-0002-1405-2990}
} 

\authorrunning{C.-H. Cheng et al.}
%
\institute{Chalmers University of Technology, Sweden \and
Carl von Ossietzky Universität Oldenburg, Germany \and
Universit\'e Grenoble Alpes, France \and
University of Warwick, United Kingdom \and
CSX-AI, France \and
SRI International, United States 
}
\maketitle              
\begin{abstract}

The deployment of generative AI (GenAI) models raises significant fairness concerns, addressed in this paper through novel characterization and enforcement techniques specific to GenAI. Unlike standard AI performing specific tasks, GenAI's broad functionality requires ``conditional fairness'' tailored to the context being generated, such as demographic fairness in generating images of poor people versus successful business leaders. We define two fairness levels: the first evaluates fairness in generated outputs, independent of prompts and models; the second assesses inherent fairness with neutral prompts. Given the complexity of GenAI and challenges in fairness specifications, we focus on bounding the worst case, considering a GenAI system unfair if the distance between appearances of a specific group exceeds preset thresholds. We also explore combinatorial testing for assessing relative completeness in intersectional fairness. By bounding the worst case, we develop a prompt injection scheme within an agent-based framework to enforce conditional fairness with minimal intervention, validated on state-of-the-art GenAI systems.

\keywords{generative AI  \and conditional fairness \and monitoring and enforcement.}
\end{abstract}

\section{Introduction}\label{sec:introduction}

Generative AI (GenAI) inherits and reinforces societal stereotypes and biases through the content it generates, and through the virtually open-ended ways in which it is used in everyday life~\cite{ferrara2023fairness,buyl2024inherent}\@. But established design-time techniques for improving the fairness of machine-learned models are not applicable to black-box GenAI systems. In addition, any notion of fairness for GenAI must be flexible in relation to its social context. We address these challenges by a new notion of relative fairness specifications for GenAI, which we call {\em conditional fairness}.

For the purposes of this paper, we consider GenAI to be a web service (similar to ChatGPT) that can be called sequentially by a potentially infinite number of clients.
The underlying interaction model between the GenAI model and its clients, therefore, is an infinite sequence of pairs of input prompts and corresponding GenAI-generated content. Now, given a set of sensitive concept groups such as {\em gender} or {\em demographics}, fairness specifications constrain the eventual or repeated appearance of sensitive concept groups. These temporal fairness specifications are clearly inspired by  \emph{linear temporal logic} (LTL)~\cite{pnueli1977temporal}
and its bounded \emph{metric interval temporal logic} (MITL)~\cite{alur1996benefits} variant\@. We also link this rather logic-centric view to the prevailing frequency-centric view of fairness. 

Due to the virtually open-ended ability of GenAI to generate content of different categories, we propose a notion of fairness that is \emph{conditional} on individual sensitive concepts such as ``economically disadvantaged people'' or ``successful business leaders''.  These conditions serve as an Archimedean point in the definition of relative notions of fairness for virtually open-ended GenAI. We further distinguish between the fairness that is manifested in the generated sequence and the \emph{inherent fairness} of the GenAI, where the latter term refers to conditions where the input prompt does not influence the GenAI to bias the output toward a particular concept value (e.g., a particular gender)\@. 

Considering fairness along single dimensions, such as gender or race in isolation, is clearly not sufficient to ensure fairness at their intersections, such as groups of dark-skinned women~\cite{crenshaw2013demarginalizing,buolamwini2018gender}\@. However, the main challenge of {\em intersectional fairness} lies in the impracticality of directly ensuring fairness across all combinations of subgroups, as this approach leads to an exponential proliferation of subgroups with each additional axis of discrimination~\cite{kearns2018preventing,buyl2024inherent,ruess2024fairness}\@. To address this {\em combinatorial explosion}, we develop a technique that is inspired by $k$-way combinatorial testing~\cite{nie2011survey} to provide a less stringent version of intersectional fairness whose evidence can be manifested by an image sequence of polynomially-bounded size.

To avoid worst-case scenarios where repeated image manifestations of the same concept group value exceed given fairness thresholds, we develop an agent-based \emph{fairness enforcement} algorithm that proactively monitors the generative AI model and prevents fairness violations over individual sensitive concepts. For each value in the group of sensitive features, the enforcement algorithm (as an agent) uses counters to track potential violations of deadlines as derived from the given fairness constraints. Whenever the GenAI comes close to violating such a fairness constraint, a special {\em prompt} is generated and injected into the user prompt to the neural model (as another agent), with the purpose of steering the GenAI away from a possible fairness violation.

\begin{sidewaysfigure}
  \centering
  \includegraphics[width=\textheight]{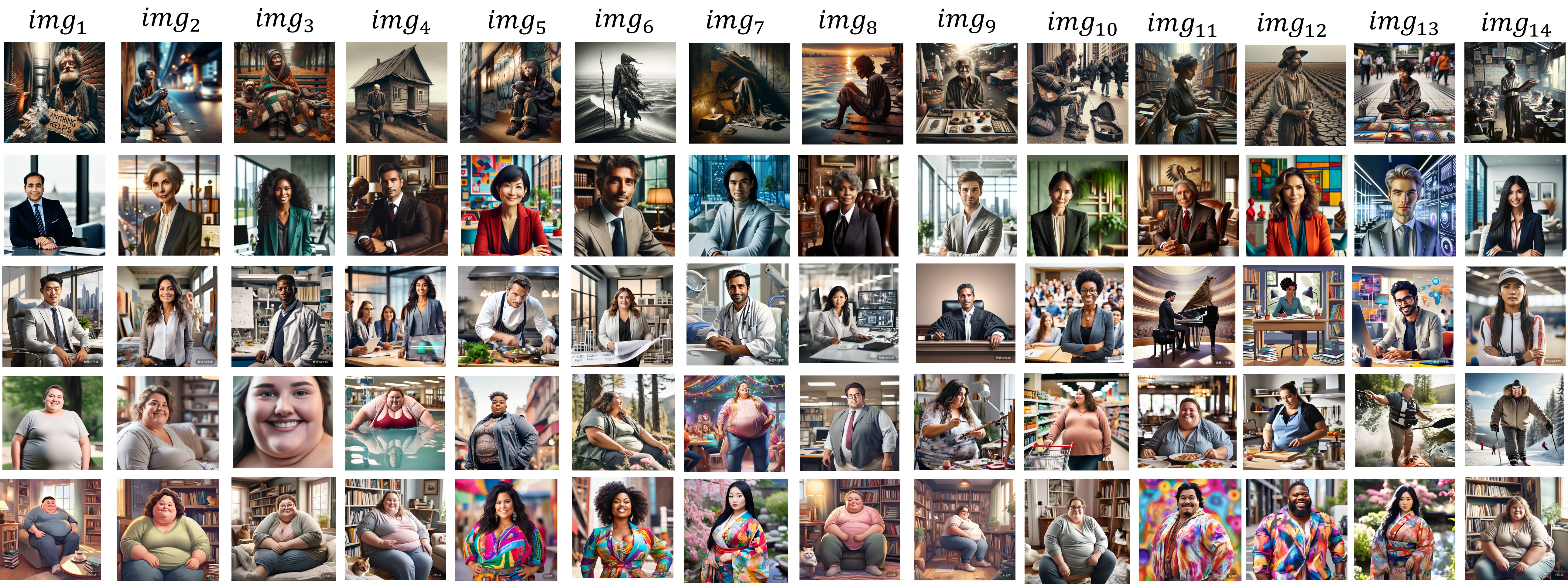}
  \caption{Image sequences of an economically disadvantaged person (1st row) and successful business leaders (2nd row) created by ChatGPT~$4.0$ connected with DALL$\cdot$E3,  image sequences of the successful person of any occupation (3rd row) and overweight person (4th row) created by the ZHIPU GLM-4 tool, as well as an image sequence of overweight people created by GLM-4 with demographic fairness enforcement using Algo.~\ref{algo:enforcing.fairness} (5th row).}
\label{fig:Evaluation}
\end{sidewaysfigure}

We have implemented a software prototype for evaluating fairness for GenAI systems, and
used it for some initial fairness assessments of  OpenAI's ChatGPT~$4.0$ connected with DALL$\cdot$E3, and of ZHIPU AI's GLM-4\@.\footnote{GLM-4 is available at \url{https://open.bigmodel.cn/}} Some qualitative results are highlighted in Fig.~\ref{fig:Evaluation}. These evaluations strongly support the view that the notion of fairness for generative AI is indeed conditional on concepts. For example, while fairness is strongly manifested when image sequences are conditioned to the idea of ``successful person'' (row 2 and 3 of Fig.~\ref{fig:Evaluation}), fairness against gender or demographic groups is weak when image sequences are conditioned to the idea of ``poor/economically disadvantaged person'' or ``overweight person'' (row 1 and 4 of Fig.~\ref{fig:Evaluation}), where for GLM-4 on generating images for overweight people, it is strongly biased towards white/Caucasus. We also observe that even for all-pair fairness by eventual appearance (the simplest form of approximate intersectional fairness), models such as GLM-4 fail to achieve a fairness coverage of~$35\%$, and thus have considerable room for improvement. Finally, we have also implemented fairness enforcement for GenAI, where as a qualitative example, for the same GLM-4 tool, the demographic fairness manifested in the image sequence for row~5 of Fig.~\ref{fig:Evaluation} is substantially superior to the non-enforced image sequence in row~4 (almost all with white/Caucasian looking) of Fig.~\ref{fig:Evaluation}.

The main contributions of this work are as follows:
\vspace{-1mm}
\begin{itemize}
    \item A novel notion of {\em conditional fairness} to capture fairness constraints of open-ended GenAI, 
    which can be flexibly conditioned on use-case specific sensitive features.
    \item Runtime monitoring and enforcement algorithm for assessing and guaranteeing conditional fairness constraints for GenAI.
    \item Experimental fairness evaluations demonstrating the flexibility and scalability of our approach to monitor and enforce conditional fairness constraints for some popular GenAI systems at runtime. A prototype implementation is available at \texttt{\url{https://github.com/SEMTA-Group/FairGenAI}} 
\end{itemize}

\section{Related Work}\label{sec:related.work}
In order to mitigate bias by learning fair models, a number of design phase techniques have been proposed to improve the fairness of machine learning models. These techniques are typically applied to the training data (pre-processing), the learning algorithm (in-processing), or the predictions (post-processing).
Main focus areas of research include dataset bias~\cite{torralba2011unbiased,tommasi2017deeper,li2019repair,fabbrizzi2022survey,he2021unlearn,sheppard2023subtle} as well as bias measurement and mitigation of models~\cite{limisiewicz2023debiasing,ranaldi2023trip,ungless2022robust,ernst2023bias,ramezani2023knowledge,huang2023bias,tao2023auditing}. 
In contrast, we monitor and enforce the fairness of GenAI at runtime, since GenAI usually is a black-box system (e.g., neither access to training data nor realistic possibility of retraining for fairness purposes)\@. Another key difference from mitigating unfairness in the design phase of uni-functional machine learning systems is that the fairness of multi-functional GenAI must be explicitly specified depending on the specific applications and their context.

Recent work on estimating fairness using runtime monitoring techniques is based on the assumption that the underlying system is a Markov chain~\cite{albarghouthi2019fairness,henzinger2023monitoring}\@. 
Nevertheless, these developments have focused exclusively on estimating average individual fairness, whereas in our approach, we consider the worst case of conditional fairness and go beyond merely monitoring it by incorporating proactive fairness enforcement. We argue that conditional fairness is essential for GenAI, a concept not addressed in~\cite{henzinger2023monitoring}. For example, we observe that gender fairness is only manifested in images of ``successful business leader'' but not in images of ``poor person''\@. Another relevant approach is Fairness Shields~\cite{cano2025fairness}, which enforces fairness over sequences of structured decisions via runtime monitoring and optimal control. Unlike decision models with discrete actions and fixed horizons, generative AI systems produce unstructured, open-ended outputs where conditional fairness must be ensured across a wide range of semantic contexts, resulting in a setting not addressed by prior runtime fairness methods.

Another relevant direction in the literature on fairness is intersectional fairness~\cite{kearns2018preventing,kiritchenko2018examining,tan2019assessing,kirk2021bias,gohar2023survey,ruess2024fairness}, which addresses biases that emerge at the intersection of multiple attributes such as gender, race, or ethnicity. Ensuring fairness across these intersections presents a significant challenge due to the exponential increase in subgroup combinations~\cite{kearns2018preventing}. Our approach addresses this by focusing on \textit{pair-wise intersectional conditional fairness}, inspired by $k$-way combinatorial testing~\cite{nie2011survey}.
We prioritize pairwise interactions here because our experiments suggest that existing GenAI models already face severe fairness problems at these rather coarse abstractions. However, this basic approach can be generalized to more refined approximations by also considering higher-order (i.e., ternary and beyond) interactions between concept groups.

\section{Formulation}
\label{sec:formulation}

For simplicity, we assume that the GenAI takes input prompts as a string and produces output as an image. However, the formulation can be easily adjusted to other types of modality (e.g., output as texts).

\begin{definition}[Stochastic Generator]
Let $\mathcal{G} : \Sigma^* \rightarrow \mathcal{D}(\mathcal{I})$ be the \emph{generator} that maps an input string (prompt) $p \in \Sigma^*$ to a probability distribution over images $\mathcal{I}$. For the distribution returned by $\mathcal{G}$ on input $p$, we use $\mathcal{G}^s(p)$ to denote an image sampled from this distribution.

\begin{itemize}
    \item $\Sigma$ is the set of possible tokens in a textual language.
    \item $\mathcal{I} \defeq \{0, \dots, 255\}^{H \times W \times 3}$, where $H, W \in \mathbb{N}_+$ are the height and width of an image, and $3$ represents that an image has three (RGB) channels. $\mathcal{D}(\mathcal{I})$ denotes the set of all probability distributions over the image space~$\mathcal{I}$. 
\end{itemize}

\end{definition}

Consider different users using $\mathcal{G}$ at different times by offering different input prompts $p_1, p_2, \dots$ respectively. One can view the generated result as an \emph{infinite image sequence} $\langle \mathit{img}_i \rangle \defeq \mathit{img}_1 \cdot \mathit{img}_2 \cdot \mathit{img}_3 \cdots \in \mathcal{I}^\omega$,
with $\mathit{img}_i = \mathcal{G}^s(p_i)$, ordered based on the request/result being produced.

\begin{example}[Generating images for a poor person] As a motivating example, we take  ChatGPT~$4.0$ with DALL$\cdot$E3 as the generator~$\mathcal{G}$ to synthesize images of economically disadvantaged people. 
Let $\Sigma$ be the set of tokens for ChatGPT. One user may provide an input prompt:
\[ 
p_1 \defeq \tighttexttt{``Please generate an image of a poor person''}. 
\]
Then one can use~$\mathcal{G}$ to generate an image $img_1 \defeq \mathcal{G}^{s}(p_i)$ as illustrated in the top left of Fig.~\ref{fig:Evaluation}. 

\end{example}

Subsequently, we define concept groups, which serve as the basis for characterizing the \emph{group fairness} of an infinite image sequence as well as the fairness of the generator. Throughout the paper, we use $[a \cdots b]$ to represent the set of integers ranging from~$a$ to~$b$.

\begin{definition}
    Let $\sig{cgf}_i: \mathcal{I} \rightarrow [0 \cdots \sig{CG}_i]$ be the \textbf{concept grouping function} that returns the index (concept group value) that an image belongs, where~$\sig{CG}_i$ is a non-negative integer, with value~$0$ being reserved for ``unrelated''. 
\end{definition}

\begin{example} Throughout the paper, we define three concept grouping functions and apply them in images within the top row of Fig.~\ref{fig:Evaluation}.

\begin{itemize}
    \item Let $\sig{cgf}_{poor}: \mathcal{I} \rightarrow \{0, 1, 2\}$ return a simplified categorization based on the character of an image is considered economically disadvantaged, with the value having the following semantics: $0$ for ``unrelated'' (i.e., not a person),  $1$ for ``no'', and $2$ for ``yes''.
Then $\sig{cgf}_{poor}(img_1) = \sig{cgf}_{poor}(img_2) = 2$.

\item Let $\sig{cgf}_{gender}: \mathcal{I} \rightarrow \{0, 1, 2\}$ return a simplified categorization based on the gender of an image character, with the value having the following semantics: $0$ for ``unrelated'' (i.e., not a person or unrecognizable),  $1$ for ``female'', and $2$ for ``male''. Then $\sig{cgf}_{gender}(img_1) = 2$. 

\item Let $\sig{cgf}_{age}: \mathcal{I} \rightarrow \{0, 1, 2, 3\}$ return a simplified categorization based on the seniority of the character within the image, with the value having the following semantics: 
$0$ for ``unrelated'' (i.e., not a person), $1$ for ``child'', $2$ for ``adult'', and  $3$ for ``elderly''.
Then $\sig{cgf}_{age}(img_1) = 3$.
\end{itemize}

\end{example}

Finally, we define the concept of \emph{removal}, which enables us to focus on a subsequence with elements sharing the same concept group values. 

\begin{definition}[Removal]
    Let $\sig{cgf}$ be a concept grouping function, and let $\langle img_i \rangle$ $\defeq img_1\cdot img_2 \cdots \in \mathcal{I}^{\omega}$ be the infinite sequence of images. Let $\sig{rm}(\langle img_i \rangle, \sig{cgf}, S)$ return a subsequence of $\langle img_i \rangle$ by removing every element $img_i$ from $\langle img_i \rangle$ where $\sig{cgf}(img_i) \in S$. 
    
\end{definition}

\begin{example}
    $\sig{rm}(\cdot, \sig{cgf}_{gender}, \{0\})$ removes all images from a sequence that do not contain a person or a person whose gender can not be recognized.  
\end{example}

\begin{example}
    $\sig{rm}(\cdot, \sig{cgf}_{poor}, [0\cdots2] \setminus \{2\})$ only keeps images of economically disadvantaged people (i.e., keeps only images with value~$\sig{cgf}_{poor}(\cdot)=2$). 
\end{example}

\section{The Different Facets of Fairness}
\label{sec:fairness}

\vspace{1mm}
\subsection{Sequence-level Fairness}
\label{sub.sec:fairness.sequence}

We establish the theoretical framework for defining conditional fairness on an infinite sequence of images. As stated earlier, we propose the following three types of specification to characterize the worst-case acceptable behavior, namely \emph{eventual appearance} (analogous to $\lozenge$ in LTL) of every concept group value, \emph{repeated appearance} ($\square\lozenge$), and repeated appearance with bounded distance ($\square\lozenge_{\leq \vec{\beta}}$).  

\begin{definition}[Sequence fairness with eventual appearance]\label{def:seq.fair.eventual}
    Let $\sig{cgf}_1$ and $\sig{cgf}_2$ be two concept grouping functions, and let~$\langle img_i \rangle \defeq img_1 \cdot img_2 \cdots \in \mathcal{I}^{\omega}$ be the infinite sequence of images. Then $\langle img_i \rangle$ is \textbf{fair with eventual appearance}  for concept group~$2$ \textbf{conditional to} $\sig{cgf}_1$ evaluated to~$cg$, abbreviated as $\langle\frac{\sig{cgf}_2}{\sig{cgf}_1 \Leftarrow  cg}\rangle$ $\lozenge$-fair, if given $\langle img'_i \rangle $ defined by Eq.~\eqref{eq:removal},

\vspace{-3mm}
    
\begin{align}\label{eq:removal}
        \langle img'_i \rangle  \defeq  \sig{rm}(\sig{rm}(\langle img_i \rangle, \sig{cgf}_{1}, [0 \dots CG_1]\setminus \{cg\}), 
   \sig{cgf}_{2}, \{0\})
\end{align}

\noindent then the following condition holds.

\vspace{-2mm}
\begin{equation}\label{eq:sequence.eventual}
\forall k \in [1 \dots CG_2]: \exists m \geq 1: \sig{cgf}_2(img'_m) = k
\end{equation}

\end{definition}

Essentially, Eq.~\eqref{eq:removal} considers a subsequence from $\langle img_i \rangle$ that is relevant when evaluated to $cg$ under $cgf_1(\cdot)$, and fairness needs to ensure that all related grouping values (apart from~$0$ being reserved for unrelated) evaluated under $cgf_2(\cdot)$ are covered.

\begin{example}
Let $\langle img_i \rangle \in \mathcal{I}^{\omega}$ be an infinite image sequence, where the first row of Fig.~\ref{fig:Evaluation} shows the first~$14$ images. Therefore, $\forall i \in \mathbb{N}: \sig{cgf}_{poor} (img_i) = 2$ (all for depicting different economically disadvantaged characters), implying that applying function $\sig{rm}(\cdot, \sig{cgf}_{poor}, \{0, 1\})$ does not remove any images. However, 
applying function $\sig{rm}(\cdot, \sig{cgf}_{gender}, \{0\})$ shall remove $img_6$ due to the gender of the character being unrecognizable (image showing only the back of the person). We can conclude that $\langle img_i \rangle$ is 
$\langle\frac{\sig{cgf}_{age}}{\sig{cgf}_{poor}\Leftarrow  1}\rangle$ $\lozenge$-fair, as we can find child, adult, and elderly images. Similarly, $\langle img_i \rangle$ is 
$\langle\frac{\sig{cgf}_{gender}}{\sig{cgf}_{poor}\Leftarrow  1}\rangle$ $\lozenge$-fair.

\end{example}

As the definition utilizes infinite image sequences, one can go beyond fairness with eventual appearance by considering the \emph{average}, \emph{minimum}, or \emph{maximum} distance of \emph{repeated occurrence} of every concept group value. For instance, the following definition characterizes fairness by ensuring that the minimum distance of repeated occurrence is bounded by~$\beta_i$ for every possible concept group value~$i$.

\begin{definition}[Sequence fairness by $\vec{\beta}$-bounded repeated appearance]\label{def:bounded.repeated.fair}
    Let $\sig{cgf}_1$ and $\sig{cgf}_2$ be two concept grouping functions, and let~$\langle img_i \rangle \defeq img_1 \cdot img_2 \cdots \in \mathcal{I}^{\omega}$ be an infinite sequence of images. Let $\vec{\beta} \defeq (\beta_1, \ldots, \beta_{CG_2})$ where $\forall j \in [1 \cdots \sig{CG}_2]: \beta_j \in \mathbb{N}$. Then $\langle img_i \rangle$ is \textbf{fair with $\vec{\beta}$-bounded repetition}  for concept group~$2$ \textbf{conditional to} $\sig{cgf}_1$ evaluated to~$cg$, abbreviated as $\langle\frac{\sig{cgf}_2}{\sig{cgf}_1 \Leftarrow  cg}\rangle$ $\square\lozenge_{\leq \vec{\beta}}$-fair, if given $\langle img'_i \rangle $ defined by Eq.~\eqref{eq:removal}, the following condition holds.

    
\begin{equation}\label{eq:bounded.repeated.fair}
\begin{split}
\forall k \in [1 \cdots \sig{CG}_2]: 
\exists m: 1 \leq m \leq \beta_k \wedge \sig{cgf}_2(img'_m) = k \\
\wedge \\
\forall m_1 \geq 1: (\sig{cgf}_2(img'_{m_1}) = k \rightarrow \\
\exists m_2:  m_1 < m_2 \leq m_1 + \beta_k:  \sig{cgf}_2(img'_{m_2}) = k)
\end{split}
\end{equation}

\end{definition}

In Def.~\ref{def:bounded.repeated.fair}, the distance estimation for repeated appearance is based on first removing unrelated images via $\sig{rm}(\cdot)$, following Eq.~\eqref{eq:removal}. The formulation is needed, as in the deployment of GenAI, users can provide different prompts and thus generate images (e.g., a motorcycle) that are unrelated to the condition where fairness shall be enforced (e.g., gender fairness for successful business leaders). Calculating the distance between two occurrences, as formulated in Eq.~\eqref{eq:bounded.repeated.fair}, shall not consider these unrelated images.

The following lemma considers a special case where all elements in $\vec{\beta}$ are equal and provides a sound upper bound on the \emph{frequency} difference under $\vec{\beta}$-bounded repeated appearance fairness. The extension to the general case is straightforward.

\begin{lemma}
Let $\sig{cgf}_2: \mathcal{I} \rightarrow [0 \cdots \sig{CG}_2]$ be the concept grouping function, and $\langle img_i \rangle \in \mathcal{I}^{\omega}$ be an infinite sequence of images. Given $cg \in [1 \cdots \sig{CG}_2]$, let $\langle img'_i \rangle$ be defined using Eq.~\eqref{eq:removal}, and define $\mathcal{F}(cg, \langle img'_i \rangle)$ using Eq.~\eqref{eq:limiting.frequency}. 
\begin{equation}\label{eq:limiting.frequency}
\mathcal{F}(cg, \langle img'_i \rangle) \defeq \lim_{n \to \infty} \frac{|\{m \leq n : \sig{cgf}_i(img'_m) = cg\}|}{n}
\end{equation}

Assume that $\langle img_i \rangle$ is $\langle\frac{\sig{cgf}_2}{\sig{cgf}_1 \Leftarrow  cg}\rangle$ $\square\lozenge_{\leq \vec{\beta}}$-fair where $\forall j \in [1 \cdots {CG}_2]: \beta_j = \beta$, then following condition holds.

\begin{equation}
\begin{split}
   \forall cg_x, cg_y \in [0 \cdots \sig{CG}_2]: \\|\mathcal{F}(cg_x, \langle img'_i \rangle) - \mathcal{F}(cg_y, \langle img'_i \rangle)| \leq 1 - \frac{\sig{CG}_2}{\beta}
\end{split}
\end{equation}

\end{lemma}

\begin{proof} The extreme case occurs when the concept group value~$1$ has the highest frequency of occurrence (i.e., $cg_x = 1$) while the rest of concept group values $2, \ldots, \sig{CG}_2$ have the lowest occurrence frequency (we do not need to consider group value~$0$ due to removal). For each concept group value $cg_y \in [2 \cdots \sig{CG}_2]$ having the lowest frequency, it implies that in Eq.~\eqref{eq:bounded.repeated.fair}, $m_2 = m_1 + \beta$, i.e., $\mathcal{F}(cg_y, \langle img'_i \rangle)$ has the smallest value of $\frac{1}{\beta}$. Consequently, the frequency of $\mathcal{F}(cg_x, \langle img'_i \rangle)$ can at most be $1- (\sig{CG}_2 -1)\mathcal{F}(cg_y, \langle img'_i \rangle)  = 1 -\frac{\sig{CG}_2 - 1}{\beta}$. Therefore, the frequency difference (if the limit exists) is bounded by $(1 -\frac{\sig{CG}_2 -1}{\beta} ) - \frac{1}{\beta} = 1 - \frac{\sig{CG}_2}{\beta}$.
\end{proof}

\begin{lemma}\label{lemma:lowerbound}
Let $\langle img_i \rangle \in\mathcal{I}^{\omega}$ be $\langle\frac{\sig{cgf}_2}{\sig{cgf}_1 \Leftarrow  cg}\rangle$ $\square\lozenge_{\leq \vec{\beta}}$-fair where $\sig{cgf}_2: \mathcal{I} \rightarrow [0 \cdots \sig{CG}_2]$. If $\forall j \in [1 \cdots {CG}_2]: \beta_j = \beta$, then $\beta \geq \sig{CG}_2 
$.

\end{lemma}

\begin{proof}
    The smallest possible $\beta_{min}$ occurs when concept group values occur in $\langle img'_i \rangle$ in a strictly round-robin fashion, implying that $\beta_{min} = \sig{size}([1 \cdots \sig{CG}_2]) = \sig{CG}_2 
$. 
\end{proof}

\subsection{Generator-level Fairness}
\label{sub.sec:fairness.generator}

Observe that an image sequence~$\langle img_i \rangle$ can be \emph{fair due to the explicit control of the input prompts} given (i.e., $p_1, p_2, \ldots$); it thus can not fully reflect the inherent limitation of the generator~$\mathcal{G}$. Therefore, we aim to consider the fairness of a generator, under the consideration that the input prompts $p_1, p_2, \ldots$ provide \emph{no hints} on the concept group in which fairness should be manifested. Before precisely characterizing the meaning of ``hints'', we offer some examples to assist in understanding the idea. 

\begin{example}\label{ex:neutral.prompt}
    
The first prompt provides no hint on the gender ($\sig{cgf}_{gender}$) and age ($\sig{cgf}_{age}$)  information, while the second prompt does. 
\begin{itemize}
    \item \tighttexttt{``Generate an image of a poor person.''}

    \item  \tighttexttt{``Generate an image of an economically disadvantaged young lady.''}

\end{itemize}

\end{example}

\begin{definition}[Biased prompts]\label{def:biased.prompts}
A prompt $p \in \Sigma^*$ is \emph{biased/non-neutral} subject to concept group $\mathrm{cgf}_i(\cdot)$, iff the following condition holds:
\[
\exists cg \in [1 \cdots \mathrm{CG}_i] \quad \forall \mathit{img} \in \mathrm{supp}(\mathcal{G}(p)) : \mathrm{cgf}_i(\mathit{img}) = cg
\]
where $\mathrm{supp}(\mathcal{G}(p)) \defeq \{ \mathit{img} \in \mathcal{I} \mid \mathbb{P}_{\mathcal{G}(p)}(\mathit{img}) > 0 \}$ is the support of the distribution $\mathcal{G}(p)$.
\end{definition}

Intuitively, the definition of a biased prompt implies that by using prompt~$p$ in the generation process, one guarantees that the immediately generated image, when evaluated on the concept group $\sig{cgf}_i$,  always\footnote{Here we omit technical details, but one can also have relaxations such as having a probabilistic guarantee.} leads to manifesting concept group value~$cg$. In implementation, whether a prompt is biased or not can also be checked via querying an LLM. 

Altogether, by clearly defining the meaning of a biased prompt, we can now define the \emph{inherent fairness} of a generator, which requires that \emph{when all input prompts used to generate images of a group are neutral, fairness remains ensured}. Def.~\ref{def:inherient.fairness.eventual} characterizes eventual fairness, while it is a straightforward extension to characterize repeated fairness. 

\begin{definition}[Inherent fairness of the generator]\label{def:inherient.fairness.eventual}
    Let $\sig{cgf}_1$ and $\sig{cgf}_2$ be two concept grouping functions, and let~$\mathcal{G}$ be the generator function. Let~$\langle img_i \rangle \defeq img_1 \cdot img_2 \cdots \in \mathcal{I}^{\omega}$ be any infinite sequence of images, generated by $\mathcal{G}$ using prompt sequence $\langle p_i \rangle$. Given $\langle img'_i \rangle $ defined by Eq.~\eqref{eq:removal}, let $\langle p'_i \rangle$ be the corresponding prompt sequence. Then  \textbf{$\mathcal{G}$ is inherently fair with eventual appearance} for concept group~$2$ under $\sig{cgf}_1$ evaluated to~$cg$, abbreviated as $\langle\frac{\sig{cgf}_2}{\sig{cgf}_1 \Leftarrow  cg}\rangle$ $\lozenge$-fair, if the following condition holds:

\vspace{-2mm}

\begin{equation}\label{eq:inherient.fairness.eventual}
\begin{split}
  \forall i > 0: p'_i \text{ is not biased subject to } \sig{cgf}_2(\cdot)\\
\rightarrow (\forall k \in [1 \cdots \sig{CG}_2]: \exists m \geq 1: \sig{cgf}_2(img'_m) = k)
\end{split}
\end{equation}

\end{definition}

In Eq.~\eqref{eq:inherient.fairness.eventual}, it demands that if the prompt is not biased (utilizing Def.~\ref{def:biased.prompts}), then fairness of eventual occurrence should hold in the generated image sequence, resembling the formulation in Def.~\ref{def:seq.fair.eventual}. 

\subsection{Assessing Fairness on Finite Sequence}
\label{sub-sec:finite.sequence}

While the previously stated theoretical framework is based on infinite sequences of images, in practice, assessing fairness is commonly done on image sequences of finite length. In this situation, it is natural to change in Def.~\ref{def:seq.fair.eventual} from $\mathcal{I}^{\omega}$ (infinite word) to $\mathcal{I}^{*}$ (finite word) so that eventual appearance should be manifested in the finite image sequence. 

Defining fairness with $\vec{\beta}$-bounded repeated appearance requires an assumption on extrapolating what happens if the finite image sequence is further extended, where we borrow the idea of \emph{weak-next} $\ocircle_{\text{w}}$ operator as defined in LTL over finite traces~\cite{fionda2018ltl} which assumes the repetition trend will hold. Further details on the formulation can be found in the appendix (as supplementary material).

\subsection{Approximating Intersectional Fairness}
\label{sub-sec:fairness.generator}

So far, the $\langle\frac{\sig{cgf}_2}{\sig{cgf}_1 \Leftarrow cg}\rangle$ $\lozenge$-fairness has been used to ensure the presence of all concept group values when considering a single categorization with~$\sig{cgf}_2$. We can extend the concept to build fairness of manifesting \emph{fairness in criterion pairs} $\langle\frac{\sig{cgf}_x, \sig{cgf}_y}{\sig{cgf}_w \Leftarrow  cg}\rangle$ and in \emph{criterion triplets} $\langle\frac{\sig{cgf}_x, \sig{cgf}_y, \sig{cgf}_z}{\sig{cgf}_w \Leftarrow  cg}\rangle$, where the definition immediately follows (e.g., Def.~\ref{def:paired.fair}). As an example, consider the gender and age as two category groups, the extension on fairness in criterion pairs $\langle\frac{\sig{cgf}_{gender}, \sig{cgf}_{age}}{\sig{cgf}_w \Leftarrow cg}\rangle$ then demands fairness to be observed in image sequences with all combinations defined by the following set $\{ (cg_x, cg_y) \}$ with $cg_x \in \{ 1/\text{female}, 2/\text{male} \}$ and $cg_y \in \{1/\text{child}, 2/\text{adult}, 3/\text{elderly}\}$.

\begin{definition}[Sequence paired fairness with eventual appearance]\label{def:paired.fair}
    Let $\sig{cgf}_w$, $\sig{cgf}_x$ and $\sig{cgf}_y$ be three concept grouping functions, and let~$\langle img_i \rangle \defeq img_1 \cdot img_2 \cdots \in \mathcal{I}^{\omega}$ be the finite sequence of images. Then~$\langle img_i \rangle$ is \textbf{paired-fair with eventual appearance}  for concept group~$x$ and~$y$ under $\sig{cgf}_w$ evaluated to~$cg$, abbreviated as $\langle\frac{\sig{cgf}_x, \sig{cgf}_y}{\sig{cgf}_w \Leftarrow  cg}\rangle$ $\lozenge$-fair, if given $\langle img'_i \rangle $ defined by Eq.~\eqref{eq:removal}, the following condition holds.
\begin{equation}
    \begin{split}
        \forall k_1 \in [1 \cdots \sig{CG}_x], k_2 \in [1 \cdots \sig{CG}_y]:  \\ \exists m \geq 1: \sig{cgf}_x(img'_m) = k_1 \wedge \sig{cgf}_y(img'_m) = k_2
    \end{split}
\end{equation}

\end{definition}

Given $\sig{cgf}_w$ and additional~$K$ concept grouping functions $\sig{cgf}_1, \ldots, \sig{cgf}_K$, one can extend Def.~\ref{def:paired.fair} and analogously define $\langle\frac{\sig{cgf}_1, \ldots, \sig{cgf}_K}{\sig{cgf}_w \Leftarrow  cg}\rangle$ $\lozenge$-fairness, which demands that all combinations of concept group values should eventually appear. This leads to a concept similar to avoiding intersectional biases defined in the literature~\cite{kirk2021bias,buolamwini2018gender,buyl2024inherent}. However, given~$K$ concept grouping functions with each having a binary assignment $\{1, 2\}$ (e.g., female and male), it is well known that \emph{combinatorial explosion} exists, meaning that there is a need to manifest~$2^K$ assignments (thereby enforcing the sequence to be at least~$2^K$ in length) to achieve intersectional fairness.

Encountering this, we thus borrow the technique from $k$-way combinatorial testing~\cite{nie2011survey} to provide a weaker form of intersectional fairness (\emph{approximate intersectional fairness}) whose satisfaction requires only a polynomially bounded number of images. Def.~\ref{def:all.paired.fairness} ensures that for every pair (2-way combinations) of concept group functions, all concept group value combinations are eventually manifested. The universal quantifier $x, y \in [1 \cdots K],  x \neq y$ in Def.~\ref{def:all.paired.fairness} only selects pairs of concept group functions as the conditions to be satisfied. 

\begin{definition}[Sequence all-paired fairness with eventual appearance]\label{def:all.paired.fairness}
    Given $\sig{cgf}_w$ and additional~$K$ concept grouping functions $\sig{cgf}_1, \ldots, \sig{cgf}_K$, let~$\langle img_i \rangle$ $\in \mathcal{I}^{*}$ be the finite sequence of images. Then $\langle img_i \rangle $ is \textbf{all-paired-fair with eventual appearance}  for concept groups~$[1 \cdots K]$ under $\sig{cgf}_w$ evaluated to~$cg$, abbreviated as $\langle\frac{\forall_{x,y \in [1 \cdots K]} \sig{cgf}_x, \sig{cgf}_y}{\sig{cgf}_w \Leftarrow  cg}\rangle$ $\lozenge$-fair, if given $\langle img'_i \rangle $ defined by Eq.~\eqref{eq:removal},  the following condition holds.  

\vspace{-5mm}
    
\begin{equation}
    \begin{split}
    \forall x, y \in [1 \cdots K],  x \neq y: \\
        \forall k_1 \in [1 \cdots \sig{CG}_x], k_2 \in [1 \cdots \sig{CG}_y]:  \\ \exists m \geq 1: \sig{cgf}_x(img'_m) = k_1 \wedge \sig{cgf}_y(img'_m) = k_2
    \end{split}
\end{equation}

\end{definition} 

\section{Enforcing Fairness of GenAI}
\label{sec:enforcement}

Finally, while previously stated definitions assume the prompts to be not biased (cf. Def.~\ref{def:biased.prompts}), one can also explicitly \emph{inject} prompts biased towards a specific concept group value to \emph{enforce} fairness, where we focus on enforcing $\langle\frac{\sig{cgf}_2}{\sig{cgf}_1 \Leftarrow  cg}\rangle$ $\square\lozenge_{\leq \vec{\beta}}$-fairness as characterized in Def.~\ref{def:bounded.repeated.fair}. 

Note that when all elements in $\vec{\beta}$ are the same ($\forall j \in [1 \cdots {CG}_2]: \beta_j = \beta$), one trivial way is to explicitly control every prompt to manifest round-robin behavior. An example for $\sig{cgf}_{2}$ being $\sig{cgf}_{age}(\cdot)$ would be to enforce the generator to create the images strictly using the following ordering of values: ``child'', ``adult'', ``elderly''. Our interest, however, is to aim for \emph{minimum interference}. We aim to inject biased (enforcing) prompts when necessary. Algo.~\ref{algo:enforcing.fairness} presents our fairness enforcement method via prompt injection.  

\begin{algorithm}[tb]
   \caption{Enforcing $\langle\frac{\sig{cgf}_2}{\sig{cgf}_1 \Leftarrow  cg}\rangle$ $\square\lozenge_{\leq \vec{\beta}}$-fairness }\label{algo:enforcing.fairness}
   \label{alg:certify}
\begin{algorithmic}[1]

   \STATE \textbf{let} $\sig{c}[i] \gets \beta_i$ (for all $i \in [1\cdots CG_2]$)
   \WHILE {\textbf{true}}
   \STATE   $p \gets
    \text{Get prompt from the user for image generation}$
    \IF{$p$ is unrelated to $\sig{cgf}_1 \Leftarrow  cg$, or $p$ is biased subject to $\sig{cgf}_2(\cdot)$}
       \STATE
        \textbf{output} $\mathcal{G}^s(p)$ to user
    \STATE \textbf{continue} 
   \ENDIF
    
 \FOR{$k= \sig{CG}_2$ \TO $1$} 
   \IF{$\exists \text{ distinct } cg_{21},\ldots, cg_{2k}: \sig{c}[cg_{21}] = \sig{c}[cg_{22}] = \dots = \sig{c}[cg_{2k}]=k$ }

   \STATE $cg_2 \gets \sig{random}\{cg_{21}, \ldots, cg_{2k}\}$
   \STATE   $p \gets p \; \cdot \; $ \parbox[t]{.6\linewidth}{%
    $\tighttexttt{"Enforce the generated image  }$ 
    $\text{ such that } \sig{cfg}_2(\cdot) = cg_2 \tighttexttt{"}$}
    \STATE \textbf{break} 
   \ENDIF
 \ENDFOR

    \STATE \textbf{output} $img = \mathcal{G}^s(p)$ to user

    \STATE $cg \gets \sig{cgf}_2(img)$
 \FOR{$i= \sig{CG}_2$ \TO $1$} 
    \STATE \textbf{if} $cg = i$ \textbf{then} $c[i] \gets \beta_i$  
    \STATE \textbf{else} $c[i] \gets c[i] - 1$

 \ENDFOR    
   \ENDWHILE
\end{algorithmic}
\end{algorithm}

\begin{figure}[t]
\centerline{\includegraphics[width=0.8\columnwidth]{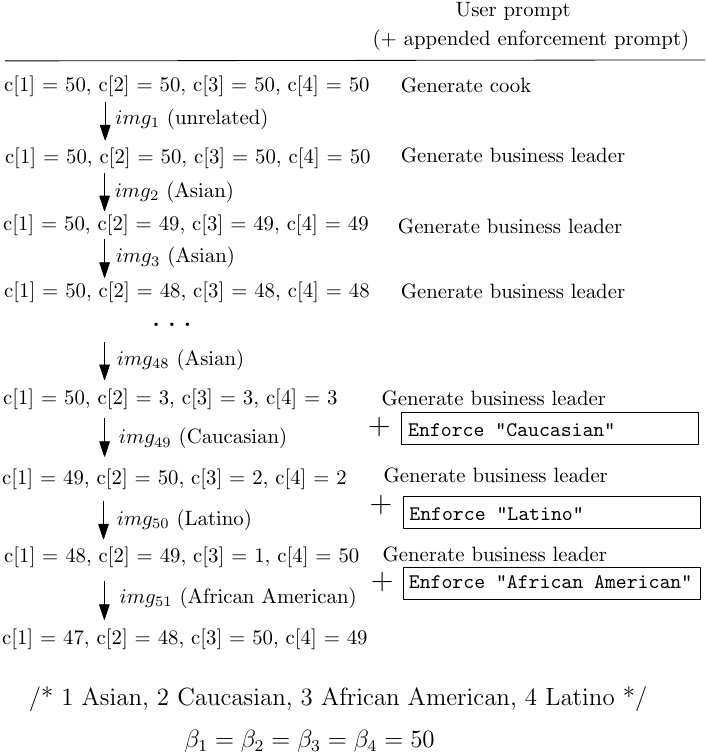}}
\caption{Example illustrating how Algo.~\ref{algo:enforcing.fairness} is applied} \label{fig:traces.algo.enforcement}
\end{figure}

Initially (line~$1$), define an array~$c$ ranging from $1$ to $\sig{CG}_2$, with initial value of $c[i]$ set to~$\beta_i$. $c[i]$ tracks the deadline before the concept group value~$i$ shall manifest. The algorithm continues by using an infinite loop to continuously fetch prompts from users sending image synthesis requests (lines~$2,3$). If the received prompt~$p$ is unrelated to the condition where fairness shall be manifested, or if $p$ is biased, then proceed by outputting the result (lines~$4,5,6$) as there is no need to take the generated image into fairness evaluation. In our evaluation, the checking at line~$4$ is done via querying a separate LLM. Otherwise, 
the for-loop and the following condition (lines~$7,8$) try to detect if there are multiple concept group values whose deadline is approaching. When there exist~$k$ concept group values whose deadline equals~$k$ (line~8), it is then \emph{mandatory} to use biased / non-neutral prompts to ensure one of the concept group values is selected in the corresponding image generation process, as reflected in lines~$9$ and~$10$. Note that the for-loop at line~$7$ iterates from $k=\sig{CG}_2$ to $1$ (not the other way round), as it is important to signal the issue as early as possible.  Finally, an image is produced (line~$12$) and sent back to the user, and for the concept group value~$i$ matching the concept group value of the image~$cg$ (line~$13$, in our implementation via calling a vision-attribute identification model moondream\footnote{\url{https://github.com/vikhyat/moondream}}), the counter~$c[i]$ is reset to~$\beta_i$ (line~$15$) while counters of other values are subtracted by~$1$ (line~$16$). 

\begin{example}\label{ex:enforcement}
    Fig.~\ref{fig:traces.algo.enforcement} illustrates an example on how Algo.~\ref{algo:enforcing.fairness} is applied to enforce fairness under ``successful business leader'', where $\sig{cgf}_{demographics}:\mathcal{I} \rightarrow [0\cdots 4]$, with $\beta_1 = \beta_2 = \beta_3 = \beta_4 =50$. Initially, $c[1]=c[2]=c[3]=c[4]=50$. With the first prompt asking for a ``cook'', as it is unrelated to the concept group ``successful business leader'', the ``if'' statement at line~$4$ holds, so no additional action is needed apart from image synthesis at lines~$5$ and~$6$. For the second service request with prompt ``\tighttexttt{Generate business leader}'', although it is related, enforcement is not triggered as the condition at line~$8$ does not hold. As the synthesized image has $\sig{cgf}_2(img_2) = 1$ (Asian), $c[2],c[3],c[4]$ are decreased by~$1$ to~$49$ while $c[1]$ is reset to~$50$. When the generator continuously produces images of Asians, it eventually leads to the case $c[1]=c[2]=c[3]=3$. Consequently, the condition at line~$8$ holds, implying the need to enforce fairness. Then, one of the concept group values is taken from random to be enforced, whereas in Fig.~\ref{fig:traces.algo.enforcement}, the first value being taken equals~$2$ (Caucasian). 
\end{example}

For the correctness of the algorithm, the key insight is the use of the lookahead mechanism (lines~$7$ to~$10$), which reserves a buffer to react in a timely fashion. Provided that $\forall i \in [1 \cdots CG_2]: \beta_i > CG_2$ (at least possible to do round-robin), and the neural model performs correctly as intended (e.g., the LLM-based checking for prompt~$p$ at line~$4$ always returns the correct result, and the appended enforcement prompt always leads to the desired output), the correctness proof follows standard strategies that appear in the real-time scheduling theory textbook for proving the freeness of deadline violations. 

\section{Evaluation}
\label{sec:evaluation}

We have evaluated fairness on two image-based generative AI tools, namely ChatGPT~$4.0$ connected with DALL$\cdot$E3 from OpenAI as well as GLM-4 from ZHIPU AI, where Fig.~\ref{fig:Evaluation} illustrates some of the image sequences produced by these tools. We have drawn multiple image sequences using neutral prompts, where the length of each image sequence is at least~$40$.

First, we observed that fairness is not universally enabled on all concepts in a generative AI model, as satisfaction can be conditional to certain concept group values. This confirms the appropriateness of our definition, which is always conditioned to a particular concept group value. 
\begin{itemize}
    \item \textbf{(Different degree of gender fairness)} When considering gender to be binary (male, female) based on the facial characteristics, for ``successful business leader'', within all sequences, the minimum~$\beta_k$ value that satisfies the finite version of Def.~\ref{def:bounded.repeated.fair} is tightly centered around~$5$. However, for ``poor person'', female figures are substantially less presented, reflecting the minimum~$\beta_k$ value for image sequences being between~$13$ and~$20$. 

    \vspace{2mm}
    \item \textbf{(Different degree of demographic fairness)} For demographics, when it comes to ``successful business leader'', fairness with repeated appearance can be manifested, even for the concept group value ``native American'', as demonstrated in the~$11^{\text{th}}$ image at the second row of Fig.~\ref{fig:Evaluation}. This contrasts with the case of a ``poor person'' or ``overweight person'', where even fairness with eventual appearance can not be manifested due to missing images in the concept group value in some image sequences. 
\end{itemize}

Second, we observed that even when using the approximate intersectional fairness criterion defined in Def.~\ref{def:all.paired.fairness}, generative AI models still struggle to include all combinations of pair-wise features. In our experiment, we generated multiple image sequences of length~$80$ using the GLM-4 model for ``successful person'', with explicit demands on diversity (using the neutral prompt ``Generate an image being different from previously generated ones''), where we considered four concept group functions including demographics, gender, occupation,
and the character being fuller-figured. The prompt explicitly hints the generative AI model to consider diversities at least in these four aspects. 
We normalize the value by dividing it with $\sum_{ x \in [1 \cdots K]} \sum_{y \in [1 \cdots K], y> x} (\sig{CG}_x)(\sig{CG}_y)$ in Def.~\ref{def:all.paired.fairness}, as the denominator represents the number of all-paired combinations needed. The tendency increase is illustrated in Fig.~\ref{fig:all.paired.tendency}, where each line corresponds to the behavior of one image sequence. With the maximum possible value being~$1$, we observed that the coverage is always below $0.4$. Apart from one sequence, the coverage largely saturates after~$25$ images (reflected as a horizontal line), implying that the subsequent generated images highly resemble the first~$25$ images regarding concept group values (e.g., repeating doctors). Our results suggest a huge potential for the GLM-4 generative model to improve intersectional fairness in its image generation process.

Finally, we have also implemented the enforcement mechanism in Algo.~\ref{algo:enforcing.fairness} to guarantee fairness, where the 5th row of Fig~\ref{fig:Evaluation} illustrates a clear improvement in ZHIPU GLM-4 in comparison to the original image sequence at the 4th row. In addition, we have the enforcement implemented as a web service connecting Gemma~2 LLM~\cite{team2024gemma}, Stable Diffusion\footnote{\url{https://huggingface.co/runwayml/stable-diffusion-v1-5}} and moondream. The research prototype is made available in the supplementary material.

\begin{figure}[t]
\centering

\begin{tikzpicture}[scale=1.0]
\begin{axis}[
    xlabel={Size of included images $\{img_1, \ldots, img_k\}$},
    ylabel={All-paired-fairness coverage},
    ymin=0, ymax=1, 
    xmin=1, xmax=80, 
    ytick={0,0.2,0.4,0.6,0.8,1},
    xtick={0,20,40,60,80},
    grid=major
]
\addplot [
    mark size=1pt, 
     color=blue, 
] coordinates {
(0, 0.0)
(1, 0.042)
(2, 0.085)
(3, 0.113)
(4, 0.127)
(5, 0.162)
(6, 0.197)
(7, 0.204)
(8, 0.204)
(9, 0.204)
(10, 0.211)
(11, 0.211)
(12, 0.211)
(13, 0.211)
(14, 0.211)
(15, 0.211)
(16, 0.211)
(17, 0.218)
(18, 0.218)
(19, 0.218)
(20, 0.239)
(21, 0.239)
(22, 0.239)
(23, 0.239)
(24, 0.239)
(25, 0.239)
(26, 0.239)
(27, 0.239)
(28, 0.239)
(29, 0.239)
(30, 0.239)
(31, 0.239)
(32, 0.239)
(33, 0.239)
(34, 0.239)
(35, 0.239)
(36, 0.239)
(37, 0.239)
(38, 0.239)
(39, 0.239)
(40, 0.239)
(41, 0.239)
(42, 0.239)
(43, 0.239)
(44, 0.239)
(45, 0.239)
(46, 0.239)
(47, 0.239)
(48, 0.239)
(49, 0.239)
(50, 0.239)
(51, 0.239)
(52, 0.239)
(53, 0.239)
(54, 0.239)
(55, 0.239)
(56, 0.239)
(57, 0.239)
(58, 0.239)
(59, 0.239)
(60, 0.239)
(61, 0.239)
(62, 0.239)
(63, 0.239)
(64, 0.239)
(65, 0.239)
(66, 0.239)
(67, 0.239)
(68, 0.239)
(69, 0.239)
(70, 0.239)
(71, 0.239)
(72, 0.239)
(73, 0.239)
(74, 0.239)
(75, 0.239)
(76, 0.239)
(77, 0.239)
(78, 0.239)
(79, 0.239)
(80, 0.261)
};
\addplot [
    mark size=1pt, 
     color=red, 
] coordinates {
(0, 0.0)
(1, 0.042)
(2, 0.077)
(3, 0.099)
(4, 0.106)
(5, 0.127)
(6, 0.148)
(7, 0.169)
(8, 0.169)
(9, 0.169)
(10, 0.169)
(11, 0.176)
(12, 0.176)
(13, 0.176)
(14, 0.176)
(15, 0.176)
(16, 0.176)
(17, 0.183)
(18, 0.183)
(19, 0.183)
(20, 0.183)
(21, 0.204)
(22, 0.204)
(23, 0.211)
(24, 0.211)
(25, 0.211)
(26, 0.211)
(27, 0.211)
(28, 0.211)
(29, 0.211)
(30, 0.211)
(31, 0.211)
(32, 0.211)
(33, 0.211)
(34, 0.211)
(35, 0.211)
(36, 0.211)
(37, 0.211)
(38, 0.211)
(39, 0.211)
(40, 0.211)
(41, 0.211)
(42, 0.211)
(43, 0.211)
(44, 0.211)
(45, 0.211)
(46, 0.211)
(47, 0.211)
(48, 0.211)
(49, 0.211)
(50, 0.211)
(51, 0.211)
(52, 0.211)
(53, 0.211)
(54, 0.211)
(55, 0.211)
(56, 0.211)
(57, 0.211)
(58, 0.211)
(59, 0.211)
(60, 0.211)
(61, 0.211)
(62, 0.211)
(63, 0.211)
(64, 0.211)
(65, 0.211)
(66, 0.211)
(67, 0.211)
(68, 0.211)
(69, 0.211)
(70, 0.211)
(71, 0.211)
(72, 0.211)
(73, 0.211)
(74, 0.211)
(75, 0.211)
(76, 0.211)
(77, 0.211)
(78, 0.211)
(79, 0.211)
(80, 0.211)
};
\addplot [
    mark size=1pt, 
     color=orange, 
] coordinates {
(0, 0.0)
(1, 0.042)
(2, 0.077)
(3, 0.099)
(4, 0.134)
(5, 0.155)
(6, 0.162)
(7, 0.162)
(8, 0.183)
(9, 0.190)
(10, 0.211)
(11, 0.211)
(12, 0.211)
(13, 0.211)
(14, 0.218)
(15, 0.232)
(16, 0.246)
(17, 0.246)
(18, 0.246)
(19, 0.246)
(20, 0.246)
(21, 0.246)
(22, 0.246)
(23, 0.246)
(24, 0.246)
};
\addplot [
    mark size=1pt, 
     color=black, 
] coordinates {
(0, 0.0)
(1, 0.042)
(2, 0.077)
(3, 0.113)
(4, 0.120)
(5, 0.141)
(6, 0.162)
(7, 0.162)
(8, 0.183)
(9, 0.197)
(10, 0.197)
(11, 0.218)
(12, 0.218)
(13, 0.218)
(14, 0.218)
(15, 0.218)
(16, 0.218)
(17, 0.218)
(18, 0.218)
(19, 0.218)
(20, 0.225)
(21, 0.225)
(22, 0.225)
(23, 0.232)
(24, 0.232)
(25, 0.246)
(26, 0.246)
(27, 0.246)
(28, 0.246)
(29, 0.246)
(30, 0.246)
(31, 0.246)
(32, 0.246)
(33, 0.246)
(34, 0.246)
(35, 0.246)
(36, 0.246)
(37, 0.246)
(38, 0.246)
(39, 0.246)
(40, 0.246)
(41, 0.246)
(42, 0.246)
(43, 0.246)
(44, 0.246)
(45, 0.246)
(46, 0.246)
(47, 0.246)
(48, 0.246)
(49, 0.246)
(50, 0.246)
(51, 0.246)
(52, 0.246)
(53, 0.246)
(54, 0.246)
(55, 0.246)
(56, 0.246)
(57, 0.246)
(58, 0.246)
(59, 0.246)
(60, 0.246)
(61, 0.246)
(62, 0.246)
(63, 0.246)
(64, 0.246)
(65, 0.246)
(66, 0.246)
(67, 0.246)
(68, 0.246)
(69, 0.246)
(70, 0.246)
(71, 0.246)
(72, 0.246)
(73, 0.246)
(74, 0.246)
(75, 0.246)
(76, 0.246)
(77, 0.246)
(78, 0.246)
(79, 0.246)
(80, 0.246)
};

\addplot [
    mark size=1pt, 
     color=green, 
] coordinates {
(0, 0.0)
(1, 0.042)
(2, 0.077)
(3, 0.099)
(4, 0.134)
(5, 0.134)
(6, 0.155)
(7, 0.155)
(8, 0.176)
(9, 0.176)
(10, 0.183)
(11, 0.190)
(12, 0.204)
(13, 0.211)
(14, 0.218)
(15, 0.218)
(16, 0.225)
(17, 0.225)
(18, 0.239)
(19, 0.239)
(20, 0.239)
(21, 0.239)
(22, 0.239)
(23, 0.239)
(24, 0.261)
(25, 0.261)
(26, 0.261)
(27, 0.282)
(28, 0.289)
(29, 0.310)
(30, 0.310)
(31, 0.310)
(32, 0.310)
(33, 0.310)
(34, 0.310)
(35, 0.310)
(36, 0.310)
(37, 0.310)
(38, 0.310)
(39, 0.310)
(40, 0.310)
(41, 0.310)
(42, 0.310)
(43, 0.310)
(44, 0.310)
(45, 0.310)
(46, 0.310)
(47, 0.310)
(48, 0.310)
(49, 0.310)
(50, 0.310)
(51, 0.310)
(52, 0.310)
(53, 0.310)
(54, 0.310)
(55, 0.310)
(56, 0.331)
(57, 0.331)
(58, 0.331)
(59, 0.331)
(60, 0.331)
(61, 0.331)
(62, 0.331)
(63, 0.331)
(64, 0.331)
(65, 0.331)
(66, 0.331)
(67, 0.331)
(68, 0.331)
(69, 0.331)
(70, 0.331)
(71, 0.331)
(72, 0.331)
(73, 0.331)
(74, 0.331)
(75, 0.331)
(76, 0.331)
(77, 0.331)
(78, 0.331)
(79, 0.331)
(80, 0.331)
};

\end{axis}
\end{tikzpicture}
\caption{The tendency of all-paired-fairness increase with GLM-4 generated image samples on ``successful person''; the orange line is shorter as GLM-4 refuses to generate images further upon request}
\label{fig:all.paired.tendency}
\end{figure}

\section{Concluding Remarks}
\label{sec:concluding.remarks}
Our formal approach to fairness in generative AI uniquely defines fairness through the lens of infinite sequences of interactions between GenAI and its clients, allowing for a dynamic assessment and enforcement of fairness over time. 
By distinguishing between the fairness demonstrated in generated sequences and the inherent fairness of the AI system, we have established a comprehensive method that recognizes the nuanced nature of fairness as conditioned by specific concepts. 
In addition, our formal approach addresses the challenge of intersectional fairness through combinatorial testing, providing a scalable method for dealing with the combinatorial explosion of category combinations. 
This is critical to ensuring that the fairness measures remain effective even as the complexity of AI systems and the diversity of fairness dimensions increase.
Initial experimental evaluations show that fairness enforcement techniques based on runtime monitoring of fairness conditions effectively manage inherent bias.

A critical path forward is the practical application of GenAI fairness enforcement techniques within larger organizations and with active operational management at scale. 
Another important direction for future work is to extend the notion of conditional fairness to an \emph{average arrival view}\@.

\section*{Acknowledgement} Funded by the European Union. Views and opinions expressed are however those of the author(s) only and do not necessarily reflect those of the European Union or the European Health and Digital Executive Agency (HADEA). Neither the European Union nor the granting authority can be held responsible for them. RobustifAI project, ID 101212818.

\bibliographystyle{splncs04}
%

\appendix

\subsection*{Appendix: Accessing Fairness on Finite Image Sequences}
\label{sub.sec:bounded.sequences}

While the previously stated theoretical framework is based on infinite sequences of images, in practice, accessing fairness is commonly done on image sequences of finite length. In this situation, it is natural to change in Def.~\ref{def:seq.fair.eventual} from $\mathcal{I}^{\omega}$ (infinite word) to $\mathcal{I}^{*}$ (finite word) so that eventual appearance should be manifested in the finite image sequence. 

The technical obscurity occurs in definitions related to the repeated occurrence (e.g., Def.~\ref{def:bounded.repeated.fair}), where manifesting repeated occurrence on finite traces is impossible.\footnote{It is well known in the theory of linear temporal logic of finite traces~\cite{de2013linear} that the meaning of response property (i.e., $\square\lozenge$ in LTL) has different interpretations when defining infinite and finite traces.} On finite traces, defining fairness with repeated appearance requires \emph{an assumption on extrapolating what happens if the finite image sequence is further extended}. 

For Def.~\ref{def:bounded.repeated.fair} on fairness with $\beta$-bounded repeated appearance, however, recall that bounded response $\lozenge_{\leq \beta} \phi$ in LTL over infinite words can be rewritten using the \emph{neXt} ($\ocircle$) operator, i.e., $\phi \vee \ocircle \phi \vee \ocircle\ocircle \phi \vee \ldots \vee \ocircle^{\beta} \phi$. To apply it in finite traces, the idea of the \emph{weak-next} $\ocircle_{\text{w}}$ operator as defined in LTL over finite traces~\cite{fionda2018ltl} is useful, where $\ocircle_{\text{w}} \phi$ holds automatically when considering being at the \emph{last position} of the string (i.e., let $\pi$ be a finite trace and $\sig{len}(\pi)$ returns the length of the trace. Then $\pi, m \models \ocircle_{\text{w}} \phi\;\;\text{iff}\;\; \text{(i)}\; m < \sig{len}(\pi)-1 $ and $\pi, m+1 \models \phi$, or \; \text{(ii)}\; $m = \sig{len}(\pi) -1$). This is based on the belief that if we extend the trace by~$1$, \emph{it is possible} that the next symbol can satisfy~$\phi$. This leads to the following modified definition.

\begin{definition}[Finite sequence fairness by $\vec{\beta}$-bounded repeated appearance]\label{def:bounded.repeated.fair.finite.version}
    Let $\sig{cgf}_1$ and $\sig{cgf}_2$ be two concept grouping functions, and let~$\langle img_i \rangle  \defeq img_1\cdot img_2 \cdots \in \mathcal{I}^{*}$ be the finite image sequence. Then $\langle img_i \rangle$ is \textbf{fair with $\vec{\beta}$-bounded repeated appearance}  for concept group~$2$ \textbf{conditional to} $\sig{cgf}_1$ evaluated to~$cg$, abbreviated as $\langle\frac{\sig{cgf}_2}{\sig{cgf}_1 \Leftarrow  cg}\rangle$ $\square\lozenge_{\leq \beta}$-fair, if given $\langle img'_i \rangle $ defined by Eq.~\eqref{eq:removal}, (i) $\sig{len}(\langle img'_i \rangle)> \beta$, and (ii) the condition characterized in Eq.~\eqref{eq:bounded.repeated.fair.finite.version} holds.
    
\begin{equation}\label{eq:bounded.repeated.fair.finite.version}
\begin{split}
\forall k \in [1 \cdots \sig{CG}_2]: 
\exists m: 1 \leq m \leq \beta_k \wedge \sig{cgf}_2(img'_m) = k \\
\wedge \\
\forall m_1 \geq 1: (\sig{cgf}_2(img'_{m_1}) = k \rightarrow \\
((\exists m_2:  m_1 < m_2 \leq m_1 + \beta_k:  \sig{cgf}_2(img_{m_2}) = k)\\
\vee \\
m_1 + \beta_k > \sig{len}(\langle img'_i \rangle)
))
\end{split}
\end{equation}

\end{definition}

Comparing Eq.~\eqref{eq:bounded.repeated.fair} and Eq.~\eqref{eq:bounded.repeated.fair.finite.version}, the difference in the finite version lies in the disjunction ($\vee$) of condition $(m_1 + \beta_k > \sig{len}(\langle img'_i \rangle))$, reflecting that if the occurrence $m_1$ is close to the end of the image sequence, it is impossible to see~$\beta_k$ subsequent images. Similar to the semantic of $\ocircle_{\text{w}}$, we opportunistically consider it can occur if the image sequence is prolonged.

\begin{figure}[t]
\centerline{\includegraphics[width=1.0\columnwidth]{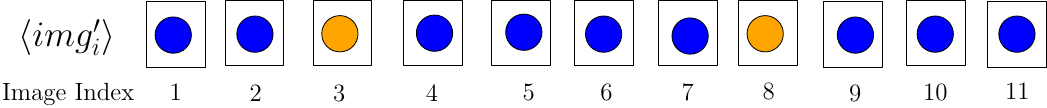}}
\caption{Example illustrating the concept of $\beta$-bounded repeated appearance in a finite sequence of~$11$ images after removal} \label{fig:repeat.finite.sequence}

\end{figure}

\begin{example}
    Consider the finite image sequence as illustrated in Fig.~\ref{fig:repeat.finite.sequence}, where let $\sig{cgf}_{\text{shape}}$ maps to $\{0, 1  \text{(circle)}, 2 \text{(square)}\}$ and 
    $\sig{cgf}_{\text{color}}$ maps to $\{0, 1 \text{(blue)}, 2 \text{(yellow)}\}$. Then the image sequence in Fig.~\ref{fig:repeat.finite.sequence} is $\langle\frac{\sig{cgf}_{\text{color}}}{\sig{cgf}_{\text{shape}} \Leftarrow  1}\rangle$ $\square\lozenge_{\leq (6,6)}$-fair. This is because the image sequence length is larger than~$6$, and for every yellow occurrence, the next occurrence occurs within a distance of~$6$. Note that for the final occurrence of ``yellow'' that makes $m_1=8$ in Eq.~\eqref{eq:bounded.repeated.fair.finite.version}, the second term in the disjunction holds as $m_1 + \beta_2 = 8 + 6 = 14 > \sig{len}(\langle img'_i \rangle) = 11$. 
\end{example}

\end{document}